\documentclass[11pt,a4paper,twocolumn,twosided,rotate]{IEEEtran} 
\usepackage{pstricks,pst-node,pst-tree,pstricks-add} 
\usepackage{amsopn}
\usepackage{amsthm}
\usepackage{spconf}
\usepackage{endnotes}
\usepackage{hyperref}
\usepackage{epsfig,psfrag}
\usepackage{pst-all}
\usepackage{amssymb,amsfonts,upref,cite,epsf,color,bm}
\usepackage{graphicx}
\usepackage{color}
\usepackage{amsmath}
\usepackage{graphicx}
\usepackage{calc}
\usepackage{booktabs}
\usepackage{tikz}
\usepackage{pgfplots}
\newcommand\defeq{:=}

\usepackage{algorithm}
\usepackage{algpseudocode}
\floatname{algorithm}{Algorithm}
\algnewcommand\algorithmicinput{\textbf{Input:}}
\algnewcommand\INPUT{\item[\algorithmicinput]}
\algnewcommand\algorithmicoutput{\textbf{Output:}}
\algnewcommand\OUTPUT{\item[\algorithmicoutput]}

\DeclareMathOperator*{\argmin}{arg\;min}

\newcommand\vect[1]{\mathbf #1}

\newcommand{\vu}{\vect{u}}

\newcommand{\vx}{\mathbf{x}}  
  
\newcommand{\vz}{z}

\newcommand{\mL}{\mathbf{L}}
\newcommand{\mW}{\mathbf{W}}
\newcommand{\measlen}{M}

\newcommand{\signalsize}{N}
\newcommand{\noise}{e}
\newcommand{\siglen}{N}

\newcommand{\graphsigs}{\mathbb{R}^{\mathcal{V}}}

\newcommand{\edges}{\mathcal{E}}
\newcommand{\cluster}{\mathcal{C}}
\newcommand{\nodes}{\mathcal{V}}
\newcommand{\graph}{\mathcal{G}}
\newcommand{\samplingset}{\mathcal{M}}

\newcommand{\edgeset}{\mathcal{S}}

\newcommand{\flow}{h}
\newcommand{\partition}{\mathcal{F}}
\newcommand{\boundary}{\partial \partition}
\newcommand{\xsig}{x[\cdot]}
\newcommand{\xsigval}[1]{x[{#1}]}

\newcommand{\estxsig}{\hat{x}[\cdot]}
\newcommand{\estxsigval}[1]{\hat{x}[{#1}]}

\newtheorem{theorem}{Theorem}
\newtheorem{definition}[theorem]{Definition}
\newtheorem{lemma}[theorem]{Lemma}

\usepackage{setspace}

\begin{document}

\title{When is Network Lasso Accurate?}
\name{Alexander Jung and Nguyen Tran Quang and Alexandru Mara}
\address{{\normalsize Department of Computer Science, Aalto University, Finland; firstname.lastname(at)aalto.fi}\\[-0.5mm]
\thanks{Parts of the work underlying this paper have been presented in \cite{MaraAsilomar2017}. 
A preprint of this manuscript is available under \url{https://arxiv.org/abs/1704.02107} \cite{WhenIsNLASSO}.}
}

\maketitle

\begin{abstract}
The ``least absolute shrinkage and selection operator'' (Lasso) method has been adapted 
recently for network-structured datasets. In particular, this network Lasso method allows to learn 
graph signals from a small number of noisy signal samples by using the total 
variation of a graph signal for regularization. While efficient and scalable implementations of the network 
Lasso are available, only little is known about the conditions on the 
underlying network structure which ensure network Lasso to be accurate. 
By leveraging concepts of compressed sensing, we address this gap and derive 
precise conditions on the underlying network topology and sampling set 
which guarantee the network Lasso for a particular loss function to deliver an accurate estimate of 
the entire underlying graph signal. We also quantify the error incurred by network Lasso in terms 
of two constants which reflect the connectivity of the sampled nodes.  
\end{abstract}


\section{Introduction}
 \label{sec_intro}
  
In many applications ranging from image processing, social networks to bioinformatics, the observed 
datasets carry an intrinsic network structure. Such datasets can be represented conveniently by 
signals defined over a ``data graph'' which models the network structure inherent to the dataset \cite{SandrMoura2014,Chen2015}. 
The nodes of this data graph represent individual data points which are labeled by some quantity 
of interest, e.g., the class membership in a classification problem. We represent this  
label information as a graph signal whose value for a particular node is given by its label 
\cite{BishopBook,SemiSupervisedBook,zhou2004regularization,SandrMoura2014,Gadde2014,AndoZhang,Belkin2006}. 
This graph signal representation of datasets allows to apply efficient methods from graph signal processing (GSP) which 
are obtained, in turn, by extending established methods (e.g., fast filtering and transforms) from discrete time signal 
processing (over chain graphs) to arbitrary graphs \cite{SandrMoura2014b,shuman2013,NarangGadde2013}. 

The resulting graph signals are typically clustered, i.e., these signals are nearly constant over well 
connected subset of nodes (clusters) in the data graph. Exploiting this clustering property enables the accurate 
recovery of graph signals from few noisy samples. In particular, using the total variation to measure how well 
a graph signal conforms with the underlying cluster structure, the authors of \cite{NetworkLasso} obtain the network Lasso (nLasso) by adapting the 
well-known Lasso estimator which is widely used for learning sparse models \cite{hastie01statisticallearning,HastieWainwrightBook}. 
The nLasso can be interpreted as an instance of the regularized empirical risk minimization principle, using total variation of a graph signal for the regularization. 
Some applications where the use of nLasso based methods has proven beneficial include  housing price prediction and personalised medicine \cite{NetworkLasso,Yamada2017}

A scalable implementation of the nLasso has been obtained via the alternating direction 
method of multipliers (ADMM) \cite{DistrOptStatistLearningADMM}. However, the authors of \cite{DistrOptStatistLearningADMM} do not 
discuss conditions on the underlying network structure which ensure success of the network Lasso. 
We close this gap in the understanding of the performance of network Lasso, by deriving sufficient 
conditions on the data graph (cluster) structure and sampling set such that nLasso is accurate. 
To this end, we introduce a simple model for clustered graph signals which are constant 
over well connected groups or clusters of nodes. We then define the notion of resolving 
sampling sets, which relates the cluster structure of the data graph to the sampling set. 
Our main contribution is an upper bound on the estimation error obtained from nLasso when applied 
to resolving sampling sets. This upper bound depends on two numerical parameters which quantify 
the connectivity between sampled nodes and cluster boundaries. 

Much of the existing work on recovery conditions and methods for graph signal recovery 
(e.g., \cite{Romero2017,Tsitsvero2016,ChenVarmaSandKov2015,ChenVarama2016,Segarra2016,WangLiuGu2015}), 
relies on spectral properties of the data graph Laplacian matrix. 
In contrast, our approach is based directly on the connectivity properties of the underlying network structure. 
The closest to our work is \cite{SharpnackJMLR2012,TrendGraph}, 
which provide sufficient conditions such that a special case of the nLasso 
(referred to as the  ``edge Lasso'') accurately recovers piece-wise constant (or clustered) graph signals 
from noisy observations. However, these works require access to fully labeled 
datasets, while we consider datasets which are only partially labeled (as it is 
typical for machine learning applications where label information is costly). 

{\bf Outline.} The problem setting considered is formalized in Section \ref{sec_setup}. 
In particular, we show how to formulate the problem of learning a clustered graph signal 
from a small amount of signal samples as a convex optimization problem, which is underlying the nLasso method. 
Our main result, i.e., an upper bound on the estimation error of nLasso is stated in Section \ref{sec_main_results}. 
Numerical experiments which illustrate our theoretical findings are discussed in Section \ref{secNumerical}. 

{\bf Notation.} We will conform to standard notation of linear algebra as used, e.g., in \cite{golub96}. 
For a binary variable $b$, we denote its negation as $\bar{b}$. 
 
\section{Problem Formulation}
\label{sec_setup}

We consider datasets which are represented by a network model, i.e., a data graph $\graph\!=\!(\nodes,\edges,\mathbf{W})$ with 
node set $\nodes = \{1,\ldots,\siglen\}$, edge set $\edges$ and weight matrix $\mathbf{W} \in \mathbb{R}_{+}^{\siglen \times \siglen}$. 
The nodes $\nodes$ of the data graph represent individual data points. For example, the node $i \!\in\! \nodes$ 
might represent a (super-)pixel in image processing, a neuron of a neural network \cite{Goodfellow-et-al-2016} or a social network user profile \cite{BigDataNetworksBook}. 

Many applications naturally suggest a notion of similarity between individual 
data points, e.g., the profiles of befriended social network users or greyscale 
values of neighbouring image pixels. 
These domain-specific notions of similarity are represented by the 
edges of the data graph $\graph$, i.e., the nodes $i,j\!\in\!\nodes$ representing similar 
data points are connected by an undirected edge $\{i,j\}\!\in\!\edges$. We denote 
the neighbourhood of the node $i\in \nodes$ by $\mathcal{N}(i) \defeq \{ j \in \nodes: \{i,j\} \in \edges \}$. 
It will be convenient to associate with each undirected edge $\{i,j\}$ a pair of directed edges, i.e., 
$(i,j)$ and $(j,i)$. With slight abuse of notation we will treat the elements of the edge set $\edges$ 
either as undirected edges $\{i,j\}$ or as pairs of two directed edges $(i,j)$ and $(j,i)$. 

In some applications it is possible to quantify the extent to which data points 
are similar, e.g., via the physical distance between neighbouring sensors in a wireless 
sensor network application \cite{ZhuRabbat2012}. 
Given two similar data points $i,j\!\in\!\nodes$, which are connected by an edge $\{i,j\} \in \edges$, 
we will quantify the strength of their connection by the edge weight $W_{i,j}\!>\!0$ which we collect 
in the symmetric weight matrix $\mathbf{W} \in \mathbb{R}_{+}^{\signalsize \times \signalsize}$. The 
absence of an edge between nodes $i,j \in \nodes$ is encoded by a zero weight $W_{i,j}\!=\!0$. 
Thus the edge structure of the data graph $\graph$ is fully specified by the support (locations of 
the non-zero entries) of the weight matrix $\mathbf{W}$. 


\subsection{Graph Signals}
\label{sec_graph_signals}

Beside the network structure, encoded in the data graph $\graph$, datasets typically also contain 
additional labeling information. 
We represent this additional label information by a graph signal defined over $\graph$. 
A graph signal $\xsig$ is a mapping $\nodes \rightarrow \mathbb{R}$, which associates every node 
$i\!\in\!\nodes$ with the signal value $\xsigval{i} \!\in\! \mathbb{R}$ (which might representing a label characterizing 
the data point). We denote the set of all graph signals defined over a graph $\graph=(\nodes,\edges,\mathbf{W})$ by 
$\graphsigs$. 

Many machine learning methods for network structured data rely on a ``cluster hypothesis'' \cite{SemiSupervisedBook}. 
In particular, we assume the graph signals $\xsig$ representing the label information of a dataset conforms 
with the cluster structure of the underlying data graph. Thus, any two nodes $i,j \in \nodes$ out of a well-connected 
region (``cluster'') of the data graph tend to have similar signal values, i.e., $\xsigval{i} \approx \xsigval{j}$. 
Two important application domains where this cluster hypothesis has been applied successfully are 
digital signal processing where time samples at adjacent time instants are strongly correlated for sufficiently 
high sampling rate (cf.\ Fig.\ \ref{fig_graph_signals}-(a)) as well as processing of natural images whose close-by 
pixels tend to be coloured likely (cf.\ Fig.\ \ref{fig_graph_signals}-(b)). The cluster hypothesis is verified also often 
in social networks where the clusters are cliques of individuals having similar properties (cf.\ Fig.\ \ref{fig_graph_signals}-(c) and \cite[Chap. 3]{NewmannBook}). 


In what follows, we quantify the extend to which a graph signal $\xsig \in \graphsigs$  conforms with the clustering structure 
of the data graph $\graph=(\nodes,\edges,\mathbf{W})$ using its \emph{total variation} (TV)
\begin{equation} 
\label{equ_def_TV}
\| \xsig \|_{\rm TV} \defeq \sum_{\{i,j\} \in \edges} W_{i,j}  | \xsigval{j}\!-\!\xsigval{i}|. 
\end{equation} 
For a subset of edges $\edgeset \subseteq \edges$, we use 
the shorthand 
\begin{equation}
\label{equ_def_shorthand_TV}
\| \xsig \|_{\edgeset} \defeq \sum_{\{i,j\} \in \edgeset} W_{i,j}  | \xsigval{j}\!-\!\xsigval{i}| . 
\end{equation}

For a supervised machine learning application, the signal values $\xsigval{i}$ might represent 
class membership in a classification problem or the target (output) value in a regression problem. 
For the house price example considered in \cite{NetworkLasso}, the vector-valued graph signal $\vx[i]$ corresponds to 
a regression weight vector for a local pricing model (used for the house market in a limited geographical 
area represented by the node $i$). 
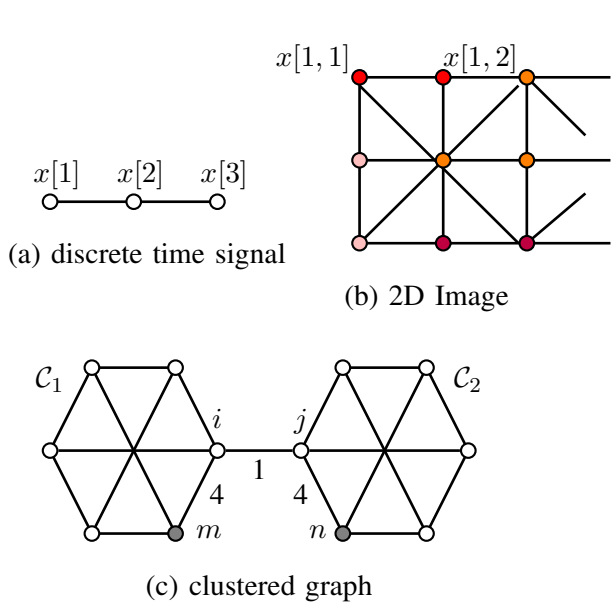
\begin{figure}
		\begin{pspicture}(-1,-3.8)(5.0,3.2)
		\psset{unit=1.1cm}
		\rput[tl](-0.2,1.5){$x[1]$}
		\rput[tl](0.8,1.5){$x[2]$}
		\rput[tl](1.8,1.5){$x[3]$}
		\psline[linewidth=1pt]{}(0,1)(1,1)(2,1)
		\pscircle[fillstyle=solid, fillcolor=white](1,1){0.1}
		\pscircle[fillstyle=solid, fillcolor=white](0,1){0.1}
		\pscircle[fillstyle=solid, fillcolor=white](2,1){0.1}
		\rput[tl](-0.5,0.5){(a) discrete time signal}
		
		
		\rput[tl](2.7,2.9){$x[1,1]$}
		\rput[tl](4.7,2.9){$x[1,2]$}
		\psline[linewidth=1pt]{}(3.7,2.5)(4.7,2.5)(5.7,2.5)(6.7,2.5)
		\psline[linewidth=1pt]{}(3.7,1.5)(4.7,1.5)(5.7,1.5)(6.7,1.5)
		\psline[linewidth=1pt]{}(3.7,0.5)(4.7,0.5)(5.7,0.5)(6.7,0.5)
		\psline[linewidth=1pt]{}(3.7,2.5)(3.7,1.5)(3.7,0.5)
		\psline[linewidth=1pt]{}(4.7,2.5)(4.7,1.5)(4.7,0.5)
		\psline[linewidth=1pt]{}(5.7,2.5)(5.7,1.5)(5.7,0.5)
		\psline[linewidth=1pt]{}(3.7,2.4)(5.6,0.5)
		\psline[linewidth=1pt]{}(3.7,0.5)(5.6,2.4)	
		
		\psline[linewidth=1pt]{}(5.7,2.5)(6.4,1.8)
		\psline[linewidth=1pt]{}(5.7,0.5)(6.4,1.1)	
		
		\pscircle[fillstyle=solid, fillcolor=red](3.7,2.5){0.1}
		\pscircle[fillstyle=solid, fillcolor=red](4.7,2.5){0.1}
		\pscircle[fillstyle=solid, fillcolor=orange](5.7,2.5){0.1}
		
		\pscircle[fillstyle=solid, fillcolor=pink](3.7,1.5){0.1}
		\pscircle[fillstyle=solid, fillcolor=orange](4.7,1.5){0.1}
		\pscircle[fillstyle=solid, fillcolor=orange](5.7,1.5){0.1}
		
		\pscircle[fillstyle=solid, fillcolor=pink](3.7,0.5){0.1}
		\pscircle[fillstyle=solid, fillcolor=purple](4.7,0.5){0.1}
		\pscircle[fillstyle=solid, fillcolor=purple](5.7,0.5){0.1}
		
		\rput[t](4.5,0.0){(b) 2D Image}
		
		
		\psline[linewidth=1pt]{}(0.5,-1.0)(1.5,-1.0)
		\psline[linewidth=1pt]{}(0.5,-3.0)(1.5,-3.0)		
		\psline[linewidth=1pt]{}(3.5,-1.0)(4.5,-1.0)
		\psline[linewidth=1pt]{}(3.5,-3.0)(4.5,-3.0)					
		\psline[linewidth=1pt]{}(0.1,-2.0)(1.9,-2.0)					
		\psline[linewidth=1pt]{}(3.1,-2.0)(4.9,-2.0)				
		\psline[linewidth=1pt]{}(2.1,-2.0)(2.9,-2.0)			
		
		\psline[linewidth=1pt]{}(0.5,-1.0)(0.0,-2.0)	
		\psline[linewidth=1pt]{}(2,-2.0)(1.5,-3.0)
		\psline[linewidth=1pt]{}(0.5,-3.0)(0.0,-2.0)	
		\psline[linewidth=1pt]{}(2,-2.0)(1.5,-1.0)		
		\psline[linewidth=1pt]{}(0.5,-1.0)(1.5,-3.0)	
		\psline[linewidth=1pt]{}(0.5,-3.0)(1.5,-1.0)						
		
		\psline[linewidth=1pt]{}(3.5,-1.0)(3.0,-2.0)	
		\psline[linewidth=1pt]{}(5,-2.0)(4.5,-3.0)
		\psline[linewidth=1pt]{}(3.5,-3.0)(3.0,-2.0)	
		\psline[linewidth=1pt]{}(5,-2.0)(4.5,-1.0)													
		\psline[linewidth=1pt]{}(3.5,-1.0)(4.5,-3.0)	
		\psline[linewidth=1pt]{}(3.5,-3.0)(4.5,-1.0)	
		
		\pscircle[fillstyle=solid, fillcolor=white](0,-2.0){0.1}
		\pscircle[fillstyle=solid, fillcolor=white](2,-2.0){0.1}
		\pscircle[fillstyle=solid, fillcolor=white](0.5,-1.0){0.1}
		\pscircle[fillstyle=solid, fillcolor=white](1.5,-1.0){0.1}
		\pscircle[fillstyle=solid, fillcolor=white](0.5,-3.0){0.1}
		\pscircle[fillstyle=solid, fillcolor=gray](1.5,-3.0){0.1}
		\pscircle[fillstyle=solid, fillcolor=white](3,-2.0){0.1}
		\pscircle[fillstyle=solid, fillcolor=white](5,-2.0){0.1}
		\pscircle[fillstyle=solid, fillcolor=white](3.5,-1.0){0.1}
		\pscircle[fillstyle=solid, fillcolor=white](4.5,-1.0){0.1}
		\pscircle[fillstyle=solid, fillcolor=gray](3.5,-3.0){0.1}
		\pscircle[fillstyle=solid, fillcolor=white](4.5,-3.0){0.1}
		
		\rput[t](0,-1.0){$\mathcal{C}_1$}
		\rput[t](5,-1.0){$\mathcal{C}_2$}
		\rput[t](2,-1.5){$i$}
		\rput[t](3,-1.5){$j$}
		\rput[t](1.9,-2.9){$m$}
		\rput[t](3.2,-2.9){$n$}
		
		\rput[t](2.5,-2.1){1}
		\rput[t](2,-2.4){4}
		\rput[t](3,-2.4){4}
		
		\rput[t](2.5,-3.5){(c) clustered graph}
		\end{pspicture}
\caption{\label{fig_graph_signals}Graph signals defined over (a) a chain graph (representing discrete time signals), 
(b) grid graph (representing 2D-images) and (c) a general graph (representing social network data), whose edges $\{i,j\}\!\in\!\edges$ 
are captioned by edge weights $W_{i,j}$.}
\end{figure} 

Consider a partition $\partition=\{\cluster_{1},\ldots,\cluster_{|\partition|}\}$ of the data graph $\graph$ 
into $|\partition|$ disjoint subsets $\cluster_{l}$ of nodes (``clusters'') such that $\nodes = \bigcup_{l=1}^{|\partition|} \cluster_{l}$. 
We associate a subset $\cluster \subseteq \nodes$ of nodes with a particular ``indicator'' graph signal 
\begin{equation}
\mathcal{I}_{\cluster}[i] \!\defeq\! \begin{cases} 1 &\mbox{ if } i \in \cluster \\ 0 & \mbox{ else.} \end{cases} 
\end{equation} 
A simple model of clustered graph signals is then obtained by piece-wise constant or clustered graph signals of the form
\begin{equation}
\label{equ_def_clustered_signal_model}
\xsigval{i} \!=\!\sum_{l =1}^{|\partition|} a_{l} \mathcal{I}_{\cluster_{l}}[i].
\end{equation} 
In Figure \ref{fig_twoclusterchain}, we depict a clustered graph signal for a chain graph with $10$ nodes which 
are partitioned into two clusters: $\cluster_{1}$ and $\cluster_{2}$. 
\begin{figure}
\begin{pspicture}(-1.5,-1)(5.5,2.5)
\psset{unit=1.2cm}
\psline[linewidth=1pt]{<->}(-0.5,0)(-0.5,0.7)
\pscircle(0,0.7){0.1}
\rput[tl](1.2,1.1){$1$}
\psline[linewidth=1pt]{}(0.1,0.7)(0.4,0.7)
\pscircle(0.5,0.7){0.1}
\psline[linewidth=1pt]{}(0.6,0.7)(0.9,0.7)
\pscircle(1,0.7){0.1}
\psline[linewidth=1pt]{}(1.1,0.7)(1.4,0.7)
\pscircle(1.5,0.7){0.1}
\psline[linewidth=1pt]{}(1.6,0.7)(1.9,0.7)
\pscircle(2,0.7){0.1}
\psline[linewidth=1pt]{}(2.05,0.75)(2.45,2.05)
\rput[tl](2.3,1.4){$1/2$}
\pscircle(2.5,2.1){0.1}
\psline[linewidth=1pt]{}(2.6,2.1)(2.9,2.1)
\pscircle(3,2.1){0.1}
\psline[linewidth=1pt]{}(3.1,2.1)(3.4,2.1)
\pscircle(3.5,2.1){0.1}
\psline[linewidth=1pt]{}(3.6,2.1)(3.9,2.1)
\pscircle(4.0,2.1){0.1}
\psline[linewidth=1pt]{}(4.1,2.1)(4.4,2.1)
\pscircle(4.5,2.1){0.1}
\psline[linewidth=1pt]{<->}(5,0)(5,2.1)

\rput[tl](-0.8,0.4){$a_1$}
\rput[tl](5.2,1.1){$a_2$}
\rput[tl](-0.3,1.1){$i\!=\!1$}
\rput[tl](4.4,2.5){$i\!=\!10$}
\rput[tl](-0.1,-0.5){$\underbrace{\hspace{2.8cm}}_{\mathcal{C}_1}$}
\rput[tl](2.4,-0.5){$\underbrace{\hspace{2.8cm}}_{\mathcal{C}_2}$}
\end{pspicture}
\caption{\label{fig_twoclusterchain}A clustered graph signal $x[i] = a_{1} \mathcal{I}_{\mathcal{C}_{1}}[i] +a_{2} \mathcal{I}_{\mathcal{C}_{2}}[i]$  
(cf.\ \eqref{equ_def_clustered_signal_model}) defined over a chain graph which is partitioned into two equal-size clusters $\cluster_{1}$ and $\cluster_{2}$ which 
consist of consecutive nodes. 
The edges connecting nodes within the same cluster have weight $1$, while the single edge connecting nodes from different clusters has weight $1/2$.}
\end{figure}
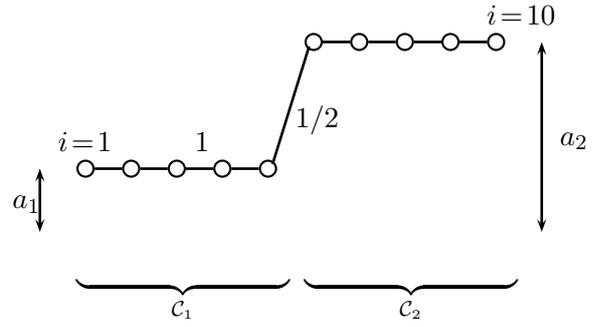 

It will be convenient to define, for a given partition $\partition$, its boundary 
$\partial \partition \subseteq \edges$ as the set of 
edges $\{i,j\} \in \edges$ which connect nodes $i \!\in\! \mathcal{C}_{a}$ and 
$j \!\in\! \mathcal{C}_{b}$ from different clusters, i.e., with $\mathcal{C}_{a} \!\neq\! \mathcal{C}_{b}$. 
With a slight abuse of notation, we will use the same symbol $\partial \partition$ also to 
denote the set of nodes which are connected to a node from another cluster. 

The TV of a clustered graph signal of the form \eqref{equ_def_clustered_signal_model} can be upper bounded as  
\begin{equation}
\label{equ_bound_TV_norm_clustered}
\| \xsig \|_{\rm TV} \leq 2 \max_{l\in\{1,\ldots,|\partition|\}} |a_{l}|  \sum_{\{i,j\} \in \boundary} W_{i,j}. 
\end{equation} 
Thus, for a partition $\partition$ with small weighted boundary $\sum_{\{i,j\} \in \boundary} W_{i,j}$, 
the associated clustered graph signals \eqref{equ_def_clustered_signal_model} have small TV $\| \xsig \|_{\rm TV}$ due to \eqref{equ_bound_TV_norm_clustered}. 


The signal model \eqref{equ_def_clustered_signal_model}, which also has been used in \cite{SharpnackJMLR2012,TrendGraph}, 
is closely related to the stochastic block model (SBM) \cite{Mossel2012}. Indeed, the SBM is obtained from 
\eqref{equ_def_clustered_signal_model} by choosing the coefficients $a_{\cluster}$ uniquely
for each cluster, i.e., $a_{\cluster} \in \{1,\ldots,|\partition|\}$. Moreover, the SBM provides a generative (stochastic) model 
for the edges within and between  the clusters $\cluster_{l}$. 

We highlight that the clustered signal model \eqref{equ_def_clustered_signal_model} is somewhat dual to the 
model of band-limited graph signals \cite{Romero2017,Romero2017,ChenVarmaSandKov2015,SemiSupervisedBook,zhou2004regularization,SandrMoura2014,Gadde2014,AndoZhang}. 
The model of band-limited graph signals is obtained by the subspaces spanned by the eigenvectors of the 
graph Laplacian corresponding to the smallest (in magnitude) eigenvalues, i.e., the low-frequency components. 
Such band-limited graph signals are smooth in the sense of small values of the Laplacian quadratic form \cite{Bapat2014}
\begin{equation}
\label{equ_Lapl_quadratic_form}
\sum_{\{i,j\} \in \edges} W_{i,j}  (\xsigval{j}\!-\! \xsigval{i})^{2} = \vx^{T} \mathbf{L} \vx. 
\end{equation}
Here, we used the vector representation $\vx=(\xsigval{1},\ldots,\xsigval{\siglen})^{T}$ of the 
graph signal $\xsig$ and the graph Laplacian matrix $\mL \in \mathbb{R}^{\siglen \times \siglen}$ defined element-wise as 
\begin{equation}
\label{equ_def_Laplacian_entry_wise}
L_{i,j} = \begin{cases} \sum_{k \in \nodes} W_{i,k} & \mbox{ if } i = j \\ - W_{i,j} & \mbox{ otherwise.} \end{cases}
\end{equation}

A band-limited graph signal $\xsig$ 
is characterized by a clustering (within a small bandwidth) of their graph Fourier transform (GFT) coefficients \cite{WangLiuGu2015}
\begin{equation} 
\label{equ_def_GFT}
\tilde{x}[l] \defeq \vu_{l}^{T} \mathbf{x} \mbox{, for } l=1,\ldots, \siglen, 
\end{equation} 
with the orthonormal eigenvectors $\{ \vu_{l} \}_{l=1}^{\siglen}$ of the graph Laplacian matrix $\mathbf{L}$. 
In particular, by the spectral decomposition of the psd graph Laplacian matrix $\mL$ (cf.\ \eqref{equ_def_Laplacian_entry_wise}), we have 
$\mathbf{L} = \mathbf{U} {\bm \Lambda} \mathbf{U}^{H}$ with $\mathbf{U} = \big(\vu_{1},\ldots,\vu_{\siglen}\big)$ and 
the diagonal matrix ${\bm \Lambda}$ having  (in decreasing order) the non-negative eigenvalues $\lambda_{l}$ of $\mathbf{L}$ on its diagonal. 

In contrast to band-limited graph signals, a clustered graph signal of the form 
\eqref{equ_def_clustered_signal_model} will typically have GFT coefficients 
which are spread out over the entire (graph) frequency range. Moreover, while band-limited 
graph signals are characterized by having a sparse GFT, a clustered graph signal of the 
form \eqref{equ_def_clustered_signal_model} has a dense (non-sparse) GFT in general. 
On the other hand, while a clustered graph signal of the form \eqref{equ_def_clustered_signal_model} has sparse signal differences 
$\{x[i]\!-\!x[j]\}_{\{i,j\} \in \edges}$, the signal differences of a band-limited graph signal are dense (non-sparse).

Let us illustrate the duality between the clustered graph signal model \eqref{equ_def_clustered_signal_model} and the model of 
band-limited graph signals (cf.\ \cite{Romero2017,Romero2017,AndoZhang}) by considering a dataset representing a 
finite length segment of a time series. The data graph $\graph_{0}$ underlying this time series data is chosen as a chain graph (cf.\ Fig.\ \ref{fig_twoclusterchain}), 
consisting of $N= 100$ nodes which represent the individual time samples. The time series is partitioned into two clusters $\cluster_{1}, \cluster_{2}$, 
each cluster consisting of $50$ consecutive nodes (time samples). We model the correlations between successive time samples using 
edge weight $W_{i,j}=1$ for data points $i,j$ belonging to the same cluster and a smaller weight $W_{i,j} = 1/2$ for the single edge $\{i,j\}$ connecting 
the two clusters $\cluster_{1}$ and $\cluster_{2}$. 

A clustered graph signal (time series) $x_{0}[i] = a_{1} \mathcal{I}_{\mathcal{C}_{1}}[i] +a_{2} \mathcal{I}_{\mathcal{C}_{2}}[i]$ (cf.\ \eqref{equ_def_clustered_signal_model} )
defined over $\graph_{0}$ 
is characterized by very sparse signal differences $\{ x_{0}[i]\!-\!x_{0}[j] \}_{\{i,j\} \in \edges}$. Indeed 
the signal difference $x_{0}[i]\!-\!\vx_{0}[j]$ of the clustered graph signal $\vx_{0}[\cdot]$ is non-zero only for the single edge 
$\{i,j\}$ which connects $\cluster_{1}$ and $\cluster_{2}$. In stark contrast, the GFT of $x_{0}[\cdot]$ is spread out over the 
entire (graph) frequency range (cf.\ Fig.\ \ref{fig_gft}), i.e., the graph signal $x_{0}[\cdot]$ does not conform with the band-limited signal model. 

On the other hand, we illustrate in Fig.\ \ref{fig_bandlimsig} a graph signal $x_{\rm BL}[\cdot]$ with GFT coefficients $\tilde{x}_{\rm BL}[l]=1$ (cf.\ \eqref{equ_def_GFT}) 
for $l=1,2$ and $\tilde{x}_{\rm BL}[l]=0$ otherwise. Thus, the graph signal is clearly band-limited (it has only two non-zero GFT coefficients) 
but the signal differences $x_{\rm BL}[i]-x_{\rm BL}[j]$ across the edges $\{i,j\} \in \edges$ are clearly non-sparse. 
\begin{figure}
\begin{center}
\includegraphics[height=3cm]{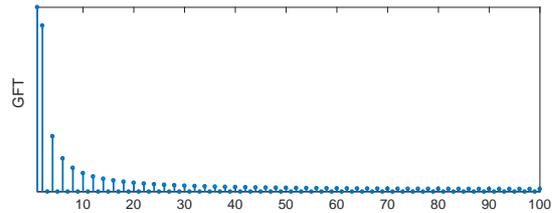}  
\caption{The magnitudes of the GFT coefficients $\tilde{x}[l]$ (cf.\ \eqref{equ_def_GFT}) of a clustered graph signal $x_{0}[\cdot]$ defined over a chain graph (cf.\ Fig.\ \ref{fig_twoclusterchain}).}
\label{fig_gft}
\end{center}
\end{figure}
\begin{figure}
\begin{center}
\includegraphics[height=3cm]{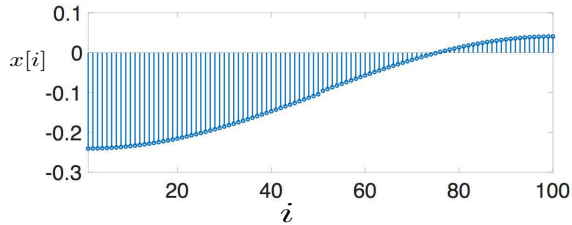}  
\caption{A strongly band-limited graph signal defined over a chain graph with $\signalsize\!=\!100$.}
\label{fig_bandlimsig}
\end{center}
\end{figure}



\subsection{Recovery via nLasso}
\label{equ_gsr_sec}

Given a dataset with data graph $\graph=(\nodes,\edges,\mW)$, we aim 
at recovering a graph signal $\xsig \in \graphsigs$ from its noisy values 
\begin{equation}
\label{equ_model_initial_labels}
y[i] = \xsigval{i} + \noise[i] 
\end{equation}  
provided on a (small) sampling set
\begin{equation} 
\label{eq_def_samplingset}
\samplingset\defeq\{i_{1},\ldots,i_{\measlen}\} \subseteq \nodes.
\end{equation}
Typically $\measlen \ll \signalsize$, i.e., the sampling set is a small subset of all nodes in 
the data graph $\graph$. 

The recovered graph signal $\estxsig$ should incur only a small empirical (or training) error  
\begin{equation}
\label{equ_def_emp_error}
\widehat{E}(\estxsig) \defeq \sum_{i \in \samplingset} | \estxsigval{i}\!-\!y[i]|.
\end{equation}
Note that the definition \eqref{equ_def_emp_error} of the empirical error involves the 
$\ell_1$-norm of the deviation $\estxsig{i}\!-\!y[i]$ between recovered and measured signal samples. 
This is different from the error criterion used in the ordinary Lasso, i.e., the squared-error loss $\sum_{i \in \samplingset} (\estxsigval{i}\!-\!y[i])^2$ \cite{BuhlGeerBook}. 
The definition \eqref{equ_def_emp_error} is beneficial for applications with 
measurement errors $\noise_{i}$ (cf.\ \eqref{equ_model_initial_labels}) having mainly small values except for a 
few large outliers \cite{Tsitsvero2016,pock_chambolle_2016}. However, by contrast 
to plain Lasso, the error function in \eqref{equ_def_emp_error} does not satisfy a restricted strong 
convexity property \cite{AgarNeg2012}, which might be 
detrimental for the convergence speed of the resulting recovery methods (cf.\ Section \ref{secNumerical}). 

In order to recover a clustered graph signal with a small TV $\| \estxsig \|_{\rm TV}$ (cf.\ \eqref{equ_bound_TV_norm_clustered}) 
from the noisy signal samples $\{y[i] \}_{i \in \samplingset}$ 
it is sensible to consider the recovery problem 
\begin{align} 
\estxsig & \in \argmin_{\tilde{x}[\cdot] \in \graphsigs} \widehat{E}(\tilde{x}[\cdot])  + \lambda \| \tilde{x}[\cdot] \|_{\rm TV}.  \label{equ_semi_sup_learning_problem}
\end{align}
This recovery problem amounts to a convex optimization problem \cite{BoydConvexBook}, which, as the notation already indicates, 
might have multiple solutions $\estxsig$ (which form a convex set). 
In what follows, we will derive conditions on 
the sampling set $\samplingset$ such that any solution $\estxsig$ of \eqref{equ_semi_sup_learning_problem} allows 
to accurately recover clustered a graph signal $\xsig$ of the form \eqref{equ_def_clustered_signal_model}. 

Any graph signal obtained from \eqref{equ_semi_sup_learning_problem} balances the empirical error $\widehat{E}(\estxsig)$ 
with the TV $ \| \estxsig \|_{\rm TV}$ in an optimal manner. The parameter $\lambda$ in \eqref{equ_semi_sup_learning_problem} 
allows to trade off a small empirical error against the amount to which the resulting signal is clustered, i.e., having a small TV. In particular, choosing a small value for $\lambda$ 
enforces the solutions of \eqref{equ_semi_sup_learning_problem} to yield a small empirical error, whereas choosing a 
large value for $\lambda$ enforces the solutions of \eqref{equ_semi_sup_learning_problem} to have small TV. 
Our analysis in Section \ref{sec_main_results} provides a selection criterion for the parameter $\lambda$ which is based on 
the location of the sampling set $\samplingset$ (cf.\ \eqref{eq_def_samplingset}) and the partition $\partition$ underlying the clustered graph signal model \eqref{equ_def_clustered_signal_model}.  
Alternatively, for sufficiently large sampling sets one might choose $\lambda$ using a cross-validation procedure \cite{hastie01statisticallearning}.

Note that the recovery problem \eqref{equ_semi_sup_learning_problem} is a particular instance of the generic 
nLasso problem studied in \cite{NetworkLasso}. There exist efficient convex optimization methods for solving the nLasso problem 
\eqref{equ_semi_sup_learning_problem} (cf.\ \cite{ZhuAugADMM} and the references therein). 
In particular, the alternating method of multipliers (ADMM) 
has been applied to the nLasso problem in \cite{NetworkLasso} to obtain a scalable learning algorithm which 
can cope with massive heterogeneous datasets.

\section{When is Network Lasso Accurate?} 
\label{sec_main_results} 

The accuracy of graph signal recovery methods based on the nLasso problem \eqref{equ_semi_sup_learning_problem}, 
depends on how close the solutions $\hat{x}[\cdot]$ of \eqref{equ_semi_sup_learning_problem} are to the true underlying graph signal $\xsig \in \graphsigs$. 
In what follows, we present a condition which guarantees any solution $\hat{x}[\cdot]$ of \eqref{equ_semi_sup_learning_problem} 
to be close to the underlying graph signal $\xsig$ if it is clustered of the form \eqref{equ_def_clustered_signal_model}. 

A main contribution of this paper is the insight that the accuracy of nLasso methods, aiming at solving \eqref{equ_semi_sup_learning_problem}, depends on the topology of the underlying data graph 
via the existence of certain \emph{flows with demands} \cite{KleinbergTardos2006}. Given a data graph $\graph$, we define a flow on it as a 
mapping $\flow[\cdot]: \nodes \times \nodes \rightarrow \mathbb{R}$ which assigns each directed edge $(i,j)$ the 
value $\flow[(i,j)]$, which can be interpreted as the amount of some quantity flowing through the edge $(i,j)$ \cite{KleinbergTardos2006}.
A flow with demands has to satisfy the conservation law 
\begin{equation} 
\label{equ_conservation_flow}
\hspace*{-6mm}\sum_{j \in \mathcal{N}(i)}\hspace*{-2mm} \flow(j,i)   - \flow(i,j)  = d[i] \mbox{, for any }  i \!\in\! \nodes
\end{equation}
with a prescribed demand $d[i]$ for each node $i \in \nodes$. 
Moreover, we require flows to satisfy the capacity constraints  
\begin{equation} 
\label{equ_cap_constraint}
|\flow (i,j)| \leq W_{i,j} \mbox{ for any  edge } (i,j) \!\in\! \edges \setminus \boundary. 
\end{equation} 
Note that the capacity constraint \eqref{equ_cap_constraint} applies only to intra-cluster edges 
and does not involve the boundary edges $\boundary$. The flow values $\flow(i,j)$ at the boundary edges $(i,j) \in \boundary$ 
take a special role in the following definition of the notion of resolving sampling sets. 
%

\begin{definition}
\label{def_sampling_set_resolves}
Consider a dataset with data graph $\graph = (\nodes, \edges,\mathbf{W})$ which contains the sampling set $\samplingset \subseteq \nodes$. 
The sampling set $\samplingset$ resolves a partition $\partition=\{\cluster_{1},\ldots,\cluster_{|\partition|}\}$ 
with constants $K$ and $L$ if, for any $b_{i,j} \in \{0,1\}$ with $\{i,j\}\!\in\! \partial \partition$, there exists a flow $\flow[\cdot]$ on 
$\graph$ (cf.\ \eqref{equ_conservation_flow}, \eqref{equ_cap_constraint}) with 
\begin{equation} 
\label{equ_resol_set_flow_cond}
\flow(i,j) = b_{i,j} \cdot L \cdot W_{i,j}\mbox{, }\flow(j,i) = \bar{b}_{i,j} \cdot L \cdot W_{i,j}
\end{equation} 
for every boundary edge $\{i,j\} \in \boundary$ and demands (cf.\ \eqref{equ_conservation_flow}) satisfying 
\begin{align}
\label{equ_bound_demand_K}
\hspace*{-3mm}|d[i]| \!\leq\! K   \mbox{ for } i \!\in\! \samplingset \mbox{, and }d[i] \!=\! 0 \mbox{ for } i \!\in\!  \nodes\!\setminus\!\samplingset. 
\end{align} 
\end{definition} 
This definition requires nodes of a resolving sampling set to be sufficiently well connected with every boundary edge $\{i,j\} \in \boundary$. In particular, 
we could think of injecting (absorbing) certain amounts of flow into (from) the data graph at the sampled nodes. At each sampled node $i\in \samplingset$, 
we can inject (absorb) a flow of level at most $K$ (cf.\ \eqref{equ_bound_demand_K}). The injected (absorbed) flow has to be routed from the sampled nodes 
via the intra-cluster edges to each boundary edge such that it carries a flow value $L \cdot W_{i,j}$. Clearly, this is only possible if there are paths of 
sufficient capacity between sampled nodes and boundary edges available. 

The definition of resolving sampling sets is quantitive as it involves the numerical constants $K$ and $L$. Our main result stated below is an 
upper bound on the estimation error of nLasso methods which depends on the value of these constants. It will turn out that resolving sampling 
sets with a small values of $K$ and large values of $L$ are beneficial for the ability of nLasso to recover the entire graph signal from noisy 
samples observed on the sampling set. However, the constants $K$ and $L$ are coupled via the flow $\flow[\cdot]$ used in Definition \ref{def_sampling_set_resolves}. 
E.g., the constant $K$ always has to satisfy 
\begin{equation} 
K \geq \max_{ \{i,j\} \in \boundary}  L W_{i,j}.
\end{equation}  
Thus, the minimum possible value for $K$ depends on the values of the edge weights $W_{i,j}$ of the data graph. 
Moreover, the minimum possible value for $L$ depends on the precise connectivity of sampled 
nodes with the boundary edges $\boundary$. Indeed, Definition \ref{def_sampling_set_resolves} requires to route 
(by satisfying the capacity constraints \eqref{equ_cap_constraint}), an amount of flow given by $L W_{i,j}$ from a 
boundary edge $\{i,j\} \in \boundary$ to the sampled nodes in $\samplingset$.

In order to make (the somewhat abstract) Definition \ref{def_sampling_set_resolves} more transparent,  
let us state an easy-to-check sufficient condition for a sampling set $\samplingset$ 
such that it resolves a given partition $\partition$. 
\begin{lemma}
\label{lem_suff_cond_NNSP}
Consider a partition $\partition=\{ \cluster_{1},\ldots,\cluster_{|\partition|} \}$ of the data graph $\graph$ which contains the 
sampling set $\samplingset \subseteq \nodes$. If each boundary edge $\{i,j\} \in \partial \partition$ 
with $i\!\in\!\mathcal{C}_{a}$, $j \!\in\! \mathcal{C}_{b}$ is connected to sampled nodes, i.e., 
$\{m,i\}\!\in\!\edges$ and $\{n,j\}\!\in\!\edges$ with $m \!\in\! \samplingset\!\cap\!\mathcal{C}_{a}$,  
$n\!\in\!\samplingset\!\cap\!\mathcal{C}_{b}$, and weights $W_{m,i}, W_{n,j} \geq L W_{i,j}$, 
then the sampling set $\samplingset$ resolves the partition $\partition$ with constants $L$ and 
\begin{equation}
K = L \cdot \max_{i \in \nodes} | \mathcal{N}(i) \cap \partial \partition|  .
\end{equation}   
\end{lemma} 
In Fig.\ \ref{fig_graph_signals}-(c) we depict a data graph consisting of two clusters 
$\partition=\{\cluster_{1},\cluster_{2}\}$. The data graph contains the sampling set $\samplingset=\{m,n\}$ 
which resolves the partition $\partition$ with constants $K=L=4$ according to Lemma \ref{lem_suff_cond_NNSP}. 

The sufficient condition provided by Lemma \ref{lem_suff_cond_NNSP} can be used to guide 
the choice for the sampling set $\samplingset$. In particular Lemma \ref{lem_suff_cond_NNSP} suggests 
to sample more densely near the boundary edges $\partial \partition$ which connect different clusters. 
This rationale allows to cope with applications where the underlying partition $\partition$ is unknown. 
In particular, we could use highly scalable local clustering methods (cf.\ \cite{Spielman_alocal}) to find the 
cluster boundaries $\partial \partition$ and then select the sampled nodes in their vicinity. 
Another approach to cope with lack of information about $\partition$ is based on using random walks to identify 
the subset of nodes with a large boundary which are sampled more densely \cite{SaeedSampta17}. 


We now state our main result which is that solutions of the nLasso problem \eqref{equ_semi_sup_learning_problem} 
allow to accurately recover the true underlying clustered graph signal $\xsig$ 
(conforming with the partition $\partition$ (cf.\ \eqref{equ_def_clustered_signal_model}) from the noisy 
measurements \eqref{equ_model_initial_labels} whenever the sampling set $\samplingset$ resolves the partition $\partition$.
\begin{theorem} 
\label{main_thm_exact_sparse}
Consider a clustered graph signal $\xsig$ of the form \eqref{equ_def_clustered_signal_model}, 
with underlying partition $\partition=\{\cluster_{1},\ldots,\cluster_{|\partition|}\}$ 
of the data graph into disjoint clusters $\cluster_{l}$. We observe the noisy signal values $y[i]$ 
at the samples nodes $\samplingset \subseteq \nodes$ (cf.\ \eqref{equ_model_initial_labels}). 
If the sampling set $\samplingset$ resolves the partition $\partition$ with parameters $K>0,L >1$, 
any solution $\estxsig$ of the nLasso problem \eqref{equ_semi_sup_learning_problem} with $\lambda \!\defeq\! 1/K$ satisfies 
\begin{equation}
\label{equ_bound_main_result}
 \| \estxsig\!-\!\xsig \|_{\rm TV}\!\leq\! (K\!+\!4/(L\!-\!1))  \sum_{i \in \samplingset} |\noise[i]|. 
\end{equation} 
\end{theorem} 
Thus, if the sampling set $\samplingset$ is chosen such that it resolves the partition $\partition=\{\cluster_{1},\ldots,\cluster_{|\partition|}\}$ 
(cf. Definition \ref{def_sampling_set_resolves}), nLasso methods (cf.\ \eqref{equ_semi_sup_learning_problem}) recover 
a clustered graph signal $\vx[\cdot]$ (cf.\ \eqref{equ_def_clustered_signal_model}) with an accuracy which is determined 
by the level of the measurement noise $\noise[i]$ (cf. \eqref{equ_model_initial_labels}).  


Let us highlight that the knowledge of the partition $\partition$ underlying the clustered graph signal model \eqref{equ_def_clustered_signal_model} 
is only needed for the analysis of nLasso methods leading to Theorem \ref{main_thm_exact_sparse}. In contrast, the actual implementation methods of nLasso methods 
based on \eqref{equ_semi_sup_learning_problem} does not require any knowledge of the underlying partition. 
What is more, if the true underlying graph signal $\vx [\cdot]$ is clustered according to \eqref{equ_def_clustered_signal_model} with different 
signal values $a_l$ for different clusters $\cluster_{l}$, the solutions of the nLasso \eqref{equ_semi_sup_learning_problem} 
could be used for determining the clusters $\cluster_{l}$ which constitute the partition $\partition$. 

We also note that the bound \eqref{equ_bound_main_result} characterizes the recovery error in terms of the semi-norm 
$\| \estxsig \!-\!\xsig \|_{\rm TV}$ which is agnostic towards a constant offset in the 
recovered graph signal $\estxsig$. In particular, having a small value of $\| \estxsig \!-\!\xsig \|_{\rm TV}$ 
does in general not imply a small squared error $\sum_{i \in \nodes}( \hat{x}[i]\!-\!\xsigval{i})^2$ as  
there might be an arbitrarily large constant offset contained in the nLasso solution $\estxsig$. 

However, if the error $\| \estxsig\!-\!\xsig \|_{\rm TV}$ is sufficiently small, we might be able to identify the boundary edges $\{i,j\} \in \boundary$ of 
the partition $\partition$ underlying a clustered graph signal of the form \eqref{equ_def_clustered_signal_model}. 

Indeed, for a clustered graph signal of the form \eqref{equ_def_clustered_signal_model}, the signal difference $\xsigval{i} - \xsigval{j}$ 
across edges is non-zero only for boundary edges $\{i,j\} \in \boundary$. Lets assume the signal differences of $\xsig$ across boundary 
edges $\{i,j\} \in \partition$ are lower bounded by some positive constant $\eta > 0$ and the nLasso error satisfies 
$\| \estxsig\!-\!\xsig \|_{\rm TV}< \eta/2$. As can be verified easily, we can then perfectly recover the boundary $\boundary$ 
of the partition $\partition = \{\cluster_{1},\ldots,\cluster_{|\partition|}\}$ as precisely those edges $\{i,j\} \in \edges$ for which $|\estxsigval{i} - \estxsigval{j}| \geq \eta/2$. 
Given the boundary $\boundary$, we can recover the partition $\partition$ and, in turn, 
average the noisy observations $y[i]$ over all sampled nodes $i \in \samplingset$ belonging to the same cluster.  
This simple post-processing of the nLasso estimate $\hat{x}[i]$ is summarized in Algorithm \ref{algo_postproc_nLasso}. 
\begin{algorithm}[h]
\caption{Post-Processing for nLasso}{}
\begin{algorithmic}[1]
\renewcommand{\algorithmicrequire}{\textbf{Input:}}
\renewcommand{\algorithmicensure}{\textbf{Output:}}
\Require  data graph $\graph=(\nodes,\edges,\mathbf{W})$, noisy signal samples $y[i]$ (cf.\ \eqref{equ_model_initial_labels}), nLasso estimate 
$\hat{x}[\cdot]$ (cf.\ \eqref{equ_semi_sup_learning_problem}) and threshold $\eta >0$
\State construct candidate boundary $\edgeset = \{ \{i,j\} \in \edges: |\hat{x}[i] - \hat{x}[j]| \geq \eta/2 \}$ 
\vspace*{2mm}
\State find partition $\widehat{\partition}= \{\cluster_{1},\ldots,\cluster_{|\widehat{\partition}|}\}$ with $\partial \widehat{\partition} = \edgeset$
\vspace*{2mm}
\State if no such partition exists return ``ERR''
\vspace*{2mm}
\State for each cluster $\cluster_{l} \in \widehat{\partition}$ 
\vspace*{2mm}
\State \quad construct set $\mathcal{A} = \cluster_{l} \cap \samplingset$
\vspace*{1mm}
\State \quad if set $\mathcal{A}$ is empty return ``ERR''
\vspace*{1mm}
\State \quad for every $i\!\in\!\cluster_{l}$ set $\tilde{x}[i]\!=\!(1/|\mathcal{A}|) \sum\limits_{j \in \mathcal{A}} y[j]$ 
\Ensure new estimate $\tilde{x}[\cdot]$ or ``ERR''
\end{algorithmic}
\label{algo_postproc_nLasso}
\end{algorithm}

\begin{lemma} 
\label{main_thm_exact_sparse_post}
Consider the setting of Theorem \ref{main_thm_exact_sparse} involving a clustered graph signal $\xsig$ 
of the form \eqref{equ_def_clustered_signal_model} with coefficients $a_{l}$ satisfying $|a_{l} - a_{l'}| > \eta$ for $l \neq l'$ with a known 
positive threshold $\eta >0$. 
We observe noisy signal samples $y[i]$ (cf.\ \eqref{equ_model_initial_labels}) over the sampling set $\samplingset$ with a bounded error 
$e[i] \leq \epsilon$. 
If the sampling set $\samplingset$ resolves the partition $\partition$ with parameters $K>0,L >1$ such that 
\begin{equation} 
(K\!+\!4/(L\!-\!1))  \sum_{i \in \samplingset} |\noise[i]| < \eta/2, 
\end{equation} 
then the signal $\tilde{x}[\cdot]$ delivered by Algorithm \ref{algo_postproc_nLasso} satisfies
\begin{equation}
\label{equ_bound_main_result_post}
\sum_{i \in \nodes} (\tilde{x}[i]\!-\!\xsigval{i})^2  \!\leq\! N \varepsilon^2. 
\end{equation} 
\end{lemma}

\section{Numerical Experiments} 
\label{secNumerical}

In order to illustrate the theoretical findings of Section \ref{sec_main_results} we report the results of some illustrative 
numerical experiments involving the recovery of clustered graph signals of the form \eqref{equ_def_clustered_signal_model} 
from a small number of noisy measurements \eqref{equ_model_initial_labels}. To this end, we implemented the iterative method 
ADMM \cite{DistrOptStatistLearningADMM} to solve the nLasso \eqref{equ_semi_sup_learning_problem} problem. 
We applied the resulting semi-supervised learning algorithm to two synthetically generated data sets. 
The first data set represents a time series, which can be represented as a graph signal over a chain graph. The nodes 
of the chain graph, which represent the discrete time instants are partitioned evenly into clusters of consecutive nodes. 
A second experiment is based on data sets generated using a recently proposed generative model for complex networks. 

\subsection{Chain Graph}
\label{sec_chain_graph}

Our first experiment, is based on a graph signal defined over a chain graph $\graph_{\rm chain}$ (cf.\ Fig.\ \ref{fig_twoclusterchain})
with $N = 10^5$ nodes $\nodes = \{1,2,\dots,N\}$, connected by $N-1$ undirected edges. 
The nodes of the data graph $\graph_{\rm chain}$ are partitioned into $N/10$ equal-sized clusters $\cluster_{l}$, $l = 1, \dots, N/10$, 
each constituted by $10$ consecutive nodes. The intrinsic clustering structure of the chain graph $\graph_{\rm chain}$ 
matches the partition $\partition_{\rm chain} = \{\cluster_l\}_{l=1}^{N/10}$ via the edge weights $W_{i,j}$. In particular, 
the weights of the edges connecting nodes within the same cluster are chosen i.i.d.\ according to $W_{i,j} \sim |\mathcal{N}(2,1/4)|$ 
(i.e., the absolute value of a Gaussian random variable with mean 2 and variance 1/4). The weights of the edges connecting
 nodes from different clusters are chosen i.i.d.\ according to $W_{i,j} \sim |\mathcal{N}(1,1/4)|$.


We then generate a clustered graph signal $x[\cdot]$ of the form \eqref{equ_def_clustered_signal_model} with coefficients 
$a_{l} \in \{1,5\}$, where the coefficients $a_{l}$ and $a_{l'}$ of consecutive clusters $\cluster_{l}$ and $\cluster_{l'}$ are different. 
The graph signal $x[\cdot]$ is observed via noisy samples $y[i]$ (cf. \eqref{equ_model_initial_labels} with $\noise[i] \sim \mathcal{N}(0,1/4)$) 
obtained for the nodes $i \in \nodes$ belonging to a sampling set $\samplingset$. We consider two different choices for the sampling set, i.e., 
$\samplingset =\samplingset_{1}$ and $\samplingset= \samplingset_{2}$. Both choices contain the same number of nodes, i.e., 
$|\samplingset_1| = |\samplingset_2| = 2 \cdot 10^4$. The sampling set $\samplingset_1$ contains neighbours of cluster 
boundaries $\partial\partition_{\rm chain}$ and conforms to Lemma \ref{lem_suff_cond_NNSP} with constants 
$K=5.39$ and $L=2$ (which have been determined numerically). 
In contrast, the sampling set $\samplingset_2$ is obtained by 
selecting nodes uniformly at random from $\nodes$ and thereby completely ignoring the cluster structure $\partition_{\rm chain}$ of $\graph_{\rm chain}$.

The noisy measurements $y[i]$ are then input to an ADMM implementation for solving the nLasso problem 
\eqref{equ_semi_sup_learning_problem} with $\lambda = 1/K$. We run ADMM for a fixed number of $300$ 
iterations and using ADMM-parameter $\rho = 0.01$ \cite{DistrOptStatistLearningADMM}. In Fig. \ref{fig_resultChain} 
we illustrate the recovered graph signals (over the first $100$ nodes of the chain graph) 
$\hat{x}[\cdot]$, obtained from noisy signal samples over either sampling set $\samplingset_{1}$ or $\samplingset_{2}$. 


As evident from Fig. \ref{fig_resultChain}, the recovered signal obtained when using the sampling set $\samplingset_1$, which takes the partition 
$\partition_{\rm chain}$ into account, better resembles the original graph signal $x[\cdot]$ than when using the randomly selected sampling set 
$\samplingset_{2}$. The favourable performance of $\samplingset_{1}$ is also reflected in the empirical normalized mean 
squared errors (NMSE) between the real and recovered graph signals, which are ${\rm NMSE}_{\samplingset_1}\!=\!3.3\cdot10^{-2}$ and ${\rm NMSE}_{\samplingset_2}\!=\!2.192\cdot 10^{-1}$, respectively.

\begin{figure}
  \centering
  \hspace*{0em}\includegraphics[width=8cm]{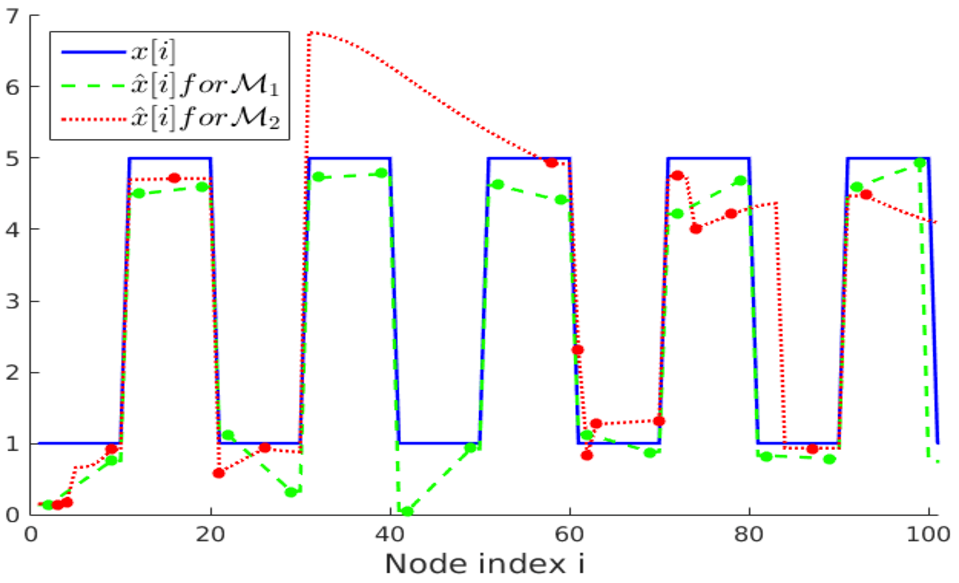}
  \caption{Clustered graph signal $\xsig$ along with the recovered graph signals obtained from sampling set 
  $\samplingset_{1}$ (Lemma \ref{lem_suff_cond_NNSP}) and $\samplingset_{2}$ (random).}
  \label{fig_resultChain}
\end{figure}

We have repeated the above experiment with the same parameters but considering noiseless initial 
samples $y[i]$ for both sampling sets $\samplingset_1$ and $\samplingset_2$. The recovered graph 
signals $\hat{x}[\cdot]$ for the first $100$ nodes of the chain are presented in Fig. \ref{fig_resultChainNoiseless}. 
It can be observed that the recovery starting from the sampling set $\samplingset_1$ (conforming to the partition 
$\partition_{\rm chain}$) perfectly resembles the original graph signal $x[\cdot]$, as expected according to our 
upper bound in \eqref{equ_bound_main_result}. 
The NMSE obtained after running ADMM for $300$ iterations for solving the nLasso problem \eqref{equ_semi_sup_learning_problem} 
are ${\rm NMSE}_{\samplingset_1}\!=\!7.5\cdot10^{-6}$ and ${\rm NMSE}_{\samplingset_2}\!=\!1.475\cdot10^{-1}$, respectively.

\begin{figure}
  \centering
  \hspace*{-1cm}\includegraphics[width=8cm]{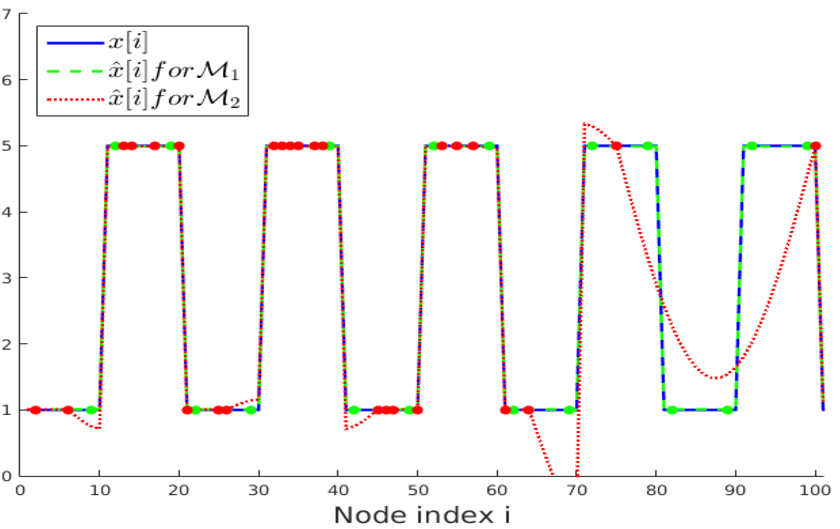} \vspace*{-0.3cm}
  \caption{Clustered graph signal $\xsig$ along with the recovered graph signals obtained from noiseless samples over sampling 
  set $\samplingset_{1}$ (Lemma \ref{lem_suff_cond_NNSP}) and $\samplingset_{2}$ (random). The noiseless signal samples $y[i] = \xsigval{i}$ are marked with dots.}
  \label{fig_resultChainNoiseless}
\end{figure}

\subsection{Complex Network}

In this second experiment, we generate a data graph $\graph_{\rm lfr}$ using the generative model 
introduced by Lancichinetti, Fortunato, and Radicchi \cite{BenchmarkComDet}, in what follows referred to as LFR model. 
The LFR model aims at imitating some key characteristics of real-world networks such as 
power law distributions of node degrees and community sizes. The data graph $\graph_{\rm lfr}$ 
contains a total of $\siglen = 10^5$ nodes which are partitioned into $1399$ clusters, $\partition_{\rm lfr}=\{\cluster_{1},\dots,\cluster_{1399}\}$. 
The nodes $\nodes$ of $\graph_{\rm lfr}$ are connected by a total of $9.45\cdot 10^5$ undirected edges $\edges$.



The edge weights $W_{i,j}$, which are also provided by the LFR model, conform to the cluster structure of $\graph_{\rm lfr}$, i.e., 
inter-cluster edges $\{i,j\} \in \edges$ with $i,j \in \cluster_l$ have larger weights compared to intra-cluster edges $\{i,j\} \in \edges$ with $i \in \cluster_l$ and  $j \in \cluster_{l'}$. 
Given the data graph $\graph_{\rm lfr}$ and partition $\partition_{\rm lfr}$ we generate a clustered graph signal according to \eqref{equ_def_clustered_signal_model} as 
$\xsigval{i} = \sum_{j=1}^{1399} a_{j} \mathcal{I}_{\cluster_{j}}[i]$ with coefficients $a_{j}$ randomly chosen i.i.d.\ according to a uniform distribution $\mathcal{U}(1,50)$. 

We then try to recover the entire graph signal $x[\cdot]$ by solving the nLasso problem \eqref{equ_semi_sup_learning_problem} using noisy measurements 
$y[i]$, according to \eqref{equ_model_initial_labels} with i.i.d.\ measurement noise $\noise[i] \sim \mathcal{N}(0,1/4)$, obtained at the nodes in a sampling set $\samplingset$. 
As in Section \ref{sec_chain_graph}, we consider two different choices  $\samplingset_{1}$ and $\samplingset_{2}$ for the sampling set which both contain the same number of nodes, i.e., 
$|\samplingset_{1}|=|\samplingset_{2}|=10^4$. 
The nodes in sampling set $\samplingset_{1}$ are selected according to Lemma \ref{lem_suff_cond_NNSP}, i.e., by choosing nodes which 
are well connected (close) to boundary edges $\partial \partition_{\rm lfr}$ which connect different clusters of the partition $\partition_{\rm lfr}$. 
In contrast, the sampling set $\samplingset_{2}$ is constructed by selecting nodes uniformly at random, i.e., the partition $\partition_{\rm lfr}$ is 
not taken into account. 
 
In order to construct the sampling set  $\samplingset_{1}$, we first sorted the edges $\{i,j\} \in \edges$ 
of the data graph $\graph_{\rm lfr}$ in ascending order according to their edge weight $W_{i,j}$. We then 
iterate over the the edges according to the list, starting with the edge having smallest weight, and for each edge $\{i,j\} \in \edges$ 
we select the neighbouring nodes of $i$ and $j$ with highest degree and add them to $\samplingset_{1}$, if they are not already included there. 
This process continues until the sampling set $\samplingset_{1}$ has reached the prescribed size of $10^4$. Using Lemma \ref{lem_suff_cond_NNSP}, 
we then verified numerically that the sampling set $\samplingset_{1}$ resolves $\partition_{\rm lfr}$ with constants  $K = 142.6$ and $L = 2$ (cf.\ Definition \ref{def_sampling_set_resolves}). 

The measurements $y[i]$ collected for each sampling sets $\samplingset_{1}$ and $\samplingset_{2}$ are fed into the ADMM 
algorithm (using parameters $\rho = 1/100$) for solving the nLasso problem \eqref{equ_semi_sup_learning_problem} with $\lambda = 1/K$. 
The evolution of the NMSE achieved by the ADMM output for an increasing number the iterations is shown in Fig.\ \ref{fig_nmseLfr}. 
According to Fig. \ref{fig_nmseLfr} the signal recovered from the sampling set $\samplingset_{1}$ approximates the true graph signal $\xsig$ 
more closely compared to when using the sampling set $\samplingset_{2}$. The NMSE achieved after $300$ iterations of ADMM 
is ${\rm NMSE}_{\samplingset_1}\!=\!1.56\cdot10^{-2}$ and ${\rm NMSE}_{\samplingset_2}\!=\!4.25\cdot10^{-2}$, respectively.

Finally, we compare the recovery accuracy of nLasso to that of plain label propagation (LP) \cite{Zhu02learningfrom}, 
which relies on a band-limited signal model (cf.\ Section \ref{sec_graph_signals}). In particular, LP quantifies signal smoothness by the Laplacian quadratic 
form \eqref{equ_Lapl_quadratic_form} instead of the total variation \eqref{equ_def_TV}, which underlies nLasso \eqref{equ_semi_sup_learning_problem}. 
The signals recovered after running the LP algorithm for $300$ iterations for the two sampling sets $\samplingset_{1}$ and $\samplingset_{2}$ incur an 
NMSE of ${\rm NMSE}_{\samplingset_1}\!=\!3.1\cdot10^{-2}$ and ${\rm NMSE}_{\samplingset_2}\!=\!7.43\cdot10^{-2}$, respectively. Thus, the signals recovered 
using nLasso are more accurate compared to LP, as illustrated in Fig. \ref{fig_resultsLfr}. However, our results indicate that LP also benefits by using the sampling set 
$\samplingset_{1}$ whose construction is guided by our theoretical findings (cf.\ Lemma \ref{lem_suff_cond_NNSP}). 

\begin{figure}
  \hspace*{0cm}\includegraphics[width=1.1\linewidth, height=6.7cm]{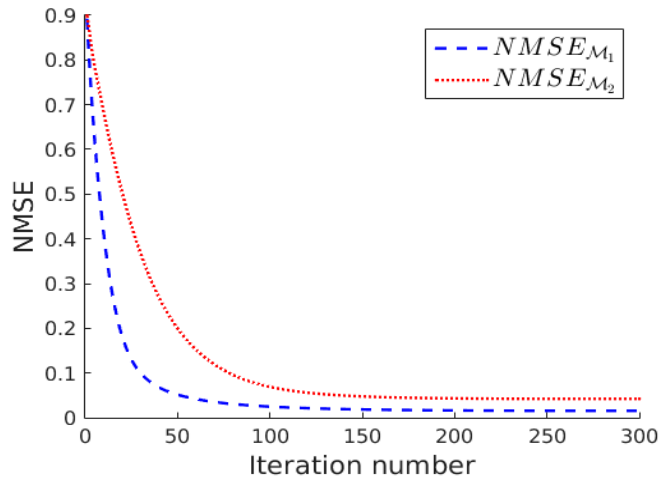} \vspace*{-0.3cm}
  \caption{Evolution of the NMSE achieved by increasing number of nLasso-ADMM iterations when using 
  sampling set $\mathcal{M}_{1}$ or $\mathcal{M}_{2}$, respectively.}
  \label{fig_nmseLfr}
\end{figure}

\begin{figure}
\centering
  \hspace*{0cm}\includegraphics[width=1.1\linewidth]{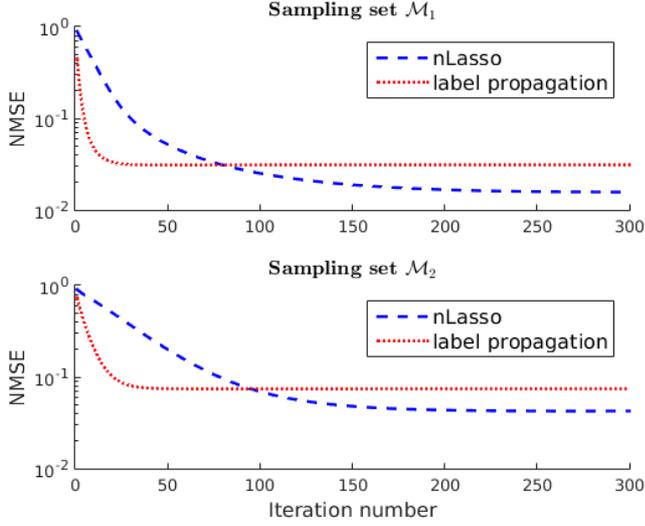} \vspace*{-0.3cm}
  \caption{Evolution of the NMSE achieved by increasing number of nLasso-ADMM iterations and LP iterations. 
  Both algorithms are fed with the signal samples obtained over sampling set $\mathcal{M}_{1}$ and $\mathcal{M}_{2}$, respectively.}
  \label{fig_resultsLfr}
\end{figure}

\section{Proofs}
\label{sec_proofs} 

The high-level idea behind the proof of Theorem \ref{main_thm_exact_sparse} 
is to adapt the concept of compatibility conditions for Lasso type estimators \cite{BuhlGeerBook}.   
This concept has been championed for analyzing Lasso type methods \cite{BuhlGeerBook}. 
Our main technical contribution is to verify the compatibility condition for a sampling set $\samplingset$ 
which resolves the partition $\partition$ underlying the signal model 
\eqref{equ_def_clustered_signal_model}  (cf.\ Lemma \ref{lem_NNSP_samplingset_suff_recovery} below).

\subsection{The Network Compatibility Condition}
As an intermediate step towards proving Theorem \ref{main_thm_exact_sparse}, 
we adopt the compatibility condition \cite{GeerBuhlConditions}, which has been introduced to analyze Lasso methods for learning sparse signals, 
to the clustered graph signal model \eqref{equ_def_clustered_signal_model}. In particular, we define the network compatibility condition for sampling graph signals 
with small total variation (cf.\ \eqref{equ_def_TV}).  
\begin{definition} 
\label{def_NNSP}
Consider a data graph $\graph=(\nodes,\edges,\mW)$ whose nodes $\nodes$ are partitioned into 
disjoint clusters $\partition =\{\cluster_{1},\ldots,\cluster_{|\partition|}\}$. 
A sampling set $\samplingset \subseteq \nodes$ is said to satisfy the network 
compatibility condition, with constants $K,L>0$, if 
\begin{equation}
\label{equ_ineq_multcompcondition_condition}
K \sum_{i \in \samplingset} |\vz[i]| +  \| \vz[\cdot] \|_{\edges \setminus \partial \partition} \geq L  \| \vz[\cdot] \|_{\partial \partition} 
\end{equation} 
for any graph signal $\vz[\cdot] \in \graphsigs$. 
\end{definition} 

It turns out that any sampling set $\samplingset$ which resolves the partition $\partition=\{\cluster_{1},\ldots,\cluster_{|\partition|}\}$ with 
constants $K$ and $L$ (cf.\ Definition \ref{def_sampling_set_resolves}) also satisfies the network compatibility condition \eqref{equ_ineq_multcompcondition_condition} 
with the same constants. 
\begin{lemma} 
\label{lem_NNSP_samplingset_suff_recovery}
Any sampling set $\samplingset$ which resolves the partition $\partition$ with parameters $K,L>0$ 
satisfies the network compatibility condition with parameters $K, L$.
\end{lemma} 
\begin{proof} 
Let us consider an arbitrary but fixed graph signal $\vz[\cdot] \in \graphsigs$. 
Since the sampling set $\samplingset$ resolves the partition $\partition$ there exists a flow $\flow[e]$ on $\graph$ with (cf.\ Definition \ref{def_sampling_set_resolves})
\begin{align} 
\label{equ_condition_flow}
&  \hspace*{0mm} \sum_{j \in \mathcal{N}(i)} \flow(j,i) - \sum_{j \in \mathcal{N}(i)} \flow(i,j) \!=\!  0  \mbox{ for all } i \notin \samplingset \nonumber \\
&  \hspace*{0mm}  \bigg| \sum_{j \in \mathcal{N}(i)} \flow(j,i) - \sum_{j \in \mathcal{N}(i)} \flow(i,j) \bigg| \!\leq\!  K  \mbox{ for all } i \in \samplingset \nonumber \\
& |\flow(i,j)| \leq W_{i,j}   \mbox{ for } (i,j)\!\notin\!\partial \partition \nonumber \\ 
   & \flow(i,j) \cdot \flow(j,i) = 0  \mbox{ for } \{ i,j \} \!\in\!\partial \partition   
\end{align}
Moreover, due to \eqref{equ_resol_set_flow_cond}, we have the important identity 
\begin{equation} 
\label{equ_crucial_ident}
 ( \flow(i,j)\!-\!\flow(j,i))(z[i]\!-\!z[j])\!=\! L W_{i,j} |z[i]\!-\!z[j]| 
\end{equation} 
which holds for all boundary edges $\{i,j\} \!\in\!\partial \partition$.
This yields, in turn, 
\begin{align}
\label{equ_inequ_nnsp_single_cluster_1}
L \| z[\cdot] \|_{\partial \partition} & \stackrel{\eqref{equ_def_shorthand_TV}}{=} \sum_{\{i,j\}  \in \partial \partition} |z[i] - z[j]| L W_{i,j} \nonumber \\ 
 & \stackrel{\eqref{equ_crucial_ident}}{=} \sum_{ (i,j)   \in \partial \partition} (z[i]\!-\!z[j])\flow(i,j). 
\end{align} 
Since $\edges\!=\!\partial \partition\!\cup\!\big( \edges\!\setminus\!\partial \partition \big)$, we can develop \eqref{equ_inequ_nnsp_single_cluster_1} as 
\begin{align}
\label{equ_inequ_nnsp_single_cluster}
L \| z[\cdot] \|_{\partial \partition} & \nonumber \\ 
& \hspace*{-15mm}  = \sum_{(i,j) \in \edges} (z[i]\!-\!z[j])\flow(i,j)\!-\!\sum_{(i,j) \in \edges \setminus \partial \partition} (z[i]\!-\!z[j])\flow(i,j) \nonumber \\ 
& \hspace*{-15mm} = \sum_{i  \in \nodes} z[i] \sum_{j \in \mathcal{N}(i)} (\flow(j,i) - \flow(i,j)) \nonumber \\ 
& \!-\!\sum_{(i,j)   \in \edges \setminus \partial \partition} (z[i]\!-\!z[j])\flow(i,j)\nonumber \\ 
& \hspace*{-15mm} \stackrel{\eqref{equ_condition_flow}}{\leq}  K  \sum_{i  \in \samplingset}  |z[i]| +  \sum_{\{i,j\} \in \edges \setminus \partial \partition} | z[i]\!-\!z[j]| W_{i,j} \nonumber \\ 
& \hspace*{-15mm} =  K  \sum_{i  \in \samplingset}  |z[i]| +    \| z[\cdot] \|_{\edges \setminus \partial \partition} 
\vspace*{-3mm}
\end{align}
which verifies \eqref{equ_ineq_multcompcondition_condition}. 
\end{proof} 

The next result shows that if the sampling set satisfies the network compatibility condition, any 
solution of the nLasso \eqref{equ_semi_sup_learning_problem} allows to accurately recover a 
clustered graph signal (cf. \eqref{equ_def_clustered_signal_model}). 
\begin{lemma} 
\label{lem_NSP1}
Consider a clustered graph signal $\xsig$ of the form \eqref{equ_def_clustered_signal_model} defined on the data graph 
$\graph=(\nodes,\edges,\mathbf{W})$ whose nodes $\nodes$ are partitioned into the clusters $\partition=\{\cluster_{1},\ldots,\cluster_{|\partition|}\}$. We observe the noisy signal values 
$y[i]$ at the sampled nodes $\samplingset \!\subseteq\!\nodes$ (cf.\ \eqref{equ_model_initial_labels}).  
If the sampling set $\samplingset$ satisfies the network compatibility condition with constants $L>1,K > 0$, then 
any solution of the nLasso problem \eqref{equ_semi_sup_learning_problem}, for the choice $\lambda \defeq 1/K$, satisfies  
\begin{equation}
 \| \hat{x}[\cdot]-\xsig \|_{\rm TV} \!\leq\! (K\!+\!4/(L\!-\!1))  \sum_{i \in \samplingset} |\noise[i]| .
\end{equation} 
\end{lemma}
\begin{proof}
Consider a solution $\hat{x}[\cdot]$ of the nLasso problem \eqref{equ_semi_sup_learning_problem} which 
is different from the true underlying clustered signal $\xsig$ (cf.\ \eqref{equ_def_clustered_signal_model}). 
We must have (cf. \eqref{equ_model_initial_labels})
\begin{equation}
\label{equ_inequ_basic_1}
\sum_{i \in \samplingset} \hspace*{-1mm} |\hat{x}[i]\!-\!y[i]|\!+\!\lambda \| \hat{x}[\cdot] \|_{\rm TV}\!\leq\!\sum_{i \in \samplingset} \hspace*{-1mm}  |\noise[i]|\!+\!\lambda \| \xsig \|_{\rm TV}
\end{equation} 
since otherwise the true underlying signal $\xsig$ would achieve a smaller objective value in 
\eqref{equ_semi_sup_learning_problem} which, in turn, would contradict the premise 
that $\hat{x}[\cdot]$ is optimal for the problem \eqref{equ_semi_sup_learning_problem}. 

Let us denote the difference between the solution $\hat{x}[\cdot]$ of \eqref{equ_semi_sup_learning_problem} and the 
true underlying clustered signal $\xsig$ by $\tilde{x}[\cdot] \defeq \hat{x}[\cdot] - \xsig$. Since $\xsig$ 
satisfies \eqref{equ_def_clustered_signal_model}, 
\begin{equation} 
\label{equ_identities_vx_edges}
\|\xsig\|_{\edges \setminus \partial \partition} = 0 \mbox{, and } 
\| \tilde{x}[\cdot] \|_{\edges \setminus \partial \partition} = \| \hat{x}[\cdot] \|_{\edges \setminus \partial \partition}.
\end{equation} 
Applying the decomposition property of the semi-norm 
$\| \cdot \|_{\rm TV}$ to \eqref{equ_inequ_basic_1} yields 
\begin{align}
\label{equ_upper_bound_complement_partition_1}
\sum_{i \in \samplingset} &|\hat{x}[i]- y[i]| + \lambda \| \hat{x}[\cdot] \|_{\edges \setminus \partial \partition} \nonumber \\ 
&\leq \sum_{i \in \samplingset} |\noise[i]| + \lambda \| \xsig \|_{\partial \partition}  - \lambda \| \hat{x}[\cdot] \|_{ \partial \partition}.
\end{align}
Therefore, using \eqref{equ_identities_vx_edges} and the triangle inequality, 
\begin{align}
\label{equ_upper_bound_complement_partition_2}
\sum_{i \in \samplingset} |\hat{x}[i]- y[i]| &+ \lambda \| \tilde{x}[\cdot] \|_{\edges \setminus \partial \partition} \nonumber\\
& \leq \lambda \| \tilde{x}[\cdot] \|_{\partial \partition}  +  \sum_{i \in \samplingset} |\noise[i]|.
\end{align}
Since $\sum_{i \in \samplingset} |\hat{x}[i]- y[i]| \geq 0 $, \eqref{equ_upper_bound_complement_partition_2} yields
\begin{align}
\label{equ_upper_bound_complement_partition}
\lambda \| \tilde{x}[\cdot] \|_{\edges \setminus \partial \partition}  \leq \lambda \| \tilde{x}[\cdot] \|_{\partial \partition}  +  \sum_{i \in \samplingset} |\noise[i]|,
\end{align}
i.e., for sufficiently small measurement noise $\noise[i]$, the signal differences of the recovery error $\tilde{x}[\cdot]\!=\!\hat{x}[\cdot]\!-\!\xsig$ 
cannot be concentrated across the edges within the clusters $\cluster_{l}$. 
Moreover, using
\begin{align}
\label{equ_upper_bound_complement_partition_3}
\sum_{i \in \samplingset} |\hat{x}[i]- y[i]| &\stackrel{\eqref{equ_model_initial_labels}}{=} \sum_{i \in \samplingset} |\hat{x}[i]- x[i] - \noise[i]| \nonumber \\
&\geq \sum_{i \in \samplingset} |\tilde{x}[i]| - \sum_{i \in \samplingset} |\noise[i]|, 
\end{align}
the inequality \eqref{equ_upper_bound_complement_partition_2} becomes 
\begin{align}
\label{equ_upper_bound_complement_partition_4}
\sum_{i \in \samplingset} |\tilde{x}[i]| + \lambda \| \tilde{x}[\cdot] \|_{\edges \setminus \partial \partition}  \leq \lambda \| \tilde{x}[\cdot] \|_{\partial \partition}  +  2 \sum_{i \in \samplingset} |\noise[i]|.
\end{align}

Thus, since the sampling set $\samplingset$ satisfies the 
network compatibility condition, we can apply \eqref{equ_ineq_multcompcondition_condition} to $\tilde{\vx}[\cdot]$ yielding 
\begin{equation} 
\label{equ_inequ_diff_signal}
\sum_{i \in \samplingset} |\tilde{x}[i] | + (1/K) \| \tilde{x}[\cdot] \|_{\edges \setminus \partial \partition} \geq (1/K)L \|\tilde{x}[\cdot]  \|_{\partial \partition}. 
\end{equation} 
Inserting \eqref{equ_inequ_diff_signal} into \eqref{equ_upper_bound_complement_partition_4}, with $\lambda = 1/K$, yields 
\begin{equation}
\label{equ_upper_bound_partial_partition}
\lambda(L-1)\| \tilde{x}[\cdot] \|_{\partial \partition} \leq 2  \sum_{i \in \samplingset} |\noise[i]|.
\vspace*{-2mm}
\end{equation}  
Combining \eqref{equ_upper_bound_complement_partition} and \eqref{equ_upper_bound_partial_partition} yields
\begin{align} 
 \| \tilde{x}[\cdot] \|_{\rm TV} & \!=\!\| \tilde{x}[\cdot] \|_{\edges \setminus \partial \partition}\!+\!\| \tilde{x}[\cdot] \|_{\partial \partition} \hspace*{-1mm} \nonumber \\
 &\hspace{-1mm} \stackrel{\eqref{equ_upper_bound_complement_partition}}{\leq}  2 \| \tilde{x}[\cdot] \|_{\partial \partition}  + (1/\lambda) \sum_{i \in \samplingset} |\noise[i]|  \nonumber \\
 &\hspace{-1mm} \stackrel{\eqref{equ_upper_bound_partial_partition}}{\leq}   \hspace*{-1mm} \frac{1\!+\!4\lambda/(L\!-\!1)}{\lambda} \hspace*{-1mm} \sum_{i \in \samplingset}\hspace*{-1mm} |\noise[i]|.
\end{align} 
\end{proof} 

\subsection{Proof of Theorem \ref{main_thm_exact_sparse}} 
\label{sec_proof_thm1}
Combine Lemma \ref{lem_NNSP_samplingset_suff_recovery} with Lemma \ref{lem_NSP1}.


\section{Conclusions}
\label{sec5_conclusion}
Given a known cluster structure of the data graph, we introduced the notion of resolving sampling sets. 
A sampling set resolves a cluster structure if there exists a sufficiently large network flow between the sampled nodes, 
with prescribed flow values over boundary edges which connect different clusters.  
Loosely speaking, this requires to choose the sampling set mainly in the boundary regions between 
different clusters in the data graph. Thus, we can leverage efficient clustering methods for identifying the 
cluster boundary regions in order to find sampling sets which resolve the intrinsic cluster structure of the network structure 
underlying a dataset. 

The verification if a particular sampling set resolves a given partition 
requires to consider all possible sign patterns for the 
boundary edges, which is intractable for large graphs. An important avenue for follow-up work is the investigation 
if resolving sampling sets can be characterized easily using probabilistic models for the underlying network structure 
and sampling sets. Moreover, we plan to extend our analysis to nLasso methods using other loss functions, e.g., 
the squared error loss and also the logistic loss function in the context of classification problems. 

\section*{Acknowledgement}
The authors are grateful to Madelon Hulsebos for a careful proof-reading of an early manuscript. 
Moreover, the constructive comments of anonymous reviewers are appreciated sincerely. 
This manuscript is available as a pre-print at the following address: \url{https://arxiv.org/abs/1704.02107}.
Copyright of this pre-print version rests with the authors

\bibliographystyle{abbrv}
\bibliography{SLPbib}

\end{document}